%% file: ms.tex
\documentclass[journal]{ieeeconf}

\usepackage{graphicx}
\usepackage{textcomp}
\usepackage{amsmath}
\usepackage{amsfonts}
\usepackage{nicefrac}
\usepackage{dsfont}
\let\labelindent\relax

\usepackage[shortlabels]{enumitem}
\usepackage{algorithm}

\usepackage{algpseudocode}
\usepackage{appendix}
\usepackage{tabularx}
\usepackage{multirow}
\usepackage{booktabs}
\usepackage{tikz}

\usepackage{xcolor}

\input{input}

\begin{document}
\title{Structured Reinforcement Learning for Incentivized Stochastic Covert Optimization}

\author{Adit Jain and Vikram Krishnamurthy, \IEEEmembership{Fellow, IEEE} 
\thanks{A. Jain and V. Krishnamurthy are with the School of Electrical and Computer Engineering, 
        Cornell University, Ithaca, NY, 14853, USA
        {email: \tt\small aj457@cornell.edu, vikramk@cornell.edu }}%
}
\maketitle
\begin{abstract}
  This { paper studies how a stochastic gradient algorithm (SG) can be controlled to hide the estimate of the local stationary point from an eavesdropper. Such problems are of significant interest in distributed optimization settings like federated learning and inventory management. A learner queries a stochastic oracle and incentivizes the oracle to obtain noisy gradient measurements and perform SG. The oracle probabilistically returns either a noisy gradient of the function} or a non-informative measurement, depending on the oracle state and incentive. The learner's query and incentive are visible to an eavesdropper who wishes to estimate the stationary point. 
This paper formulates the problem of the learner performing covert optimization by dynamically incentivizing the stochastic oracle and obfuscating the eavesdropper as a finite-horizon Markov decision process (MDP). Using conditions for interval-dominance on the cost and transition probability structure, we show that the optimal policy for the MDP has a monotone threshold structure. We propose searching for the optimal stationary policy with the threshold structure using a stochastic approximation algorithm and a multi-armed bandit approach. The effectiveness of our methods is numerically demonstrated on a covert federated learning hate-speech classification task.
\end{abstract}

\section{Introduction}
{The learner aims to obtain an estimate $\learnerestimate$ for a point $\optimalpoint \in \argmin_{\functionvar\in\R^\dimfunc} \function(\functionvar)$~\footnote{By $\argmin$ or minimizer we mean a local stationary point of $\function \in \mathcal{C}^2$.} by querying a stochastic oracle. At each time $\timeindex = 1, 2, \dots$, the learner sends query $\query_\timeindex \in \R^\dimfunc$ and incentive $\price_\timeindex$ to a stochastic oracle in state $\oraclestate_\timeindex$. The oracle returns a noisy gradient, $\response_\timeindex$ evaluated at $\query_\timeindex$ as follows: }
\begin{align}\label{eq:oraclereply}
    \response_\timeindex = \begin{cases}
            \gradient\function(\query_\timeindex) + \noise_\timeindex & \text{with prob.} \ \succfunction(\oraclestate_\timeindex,\price_\timeindex) \\
            {\textbf{0}}\ \text{(non-informative)} &\text{with prob.}\ 1-\succfunction(\oraclestate_\timeindex,\price_\timeindex) 
        \end{cases}.
\end{align}
{Here $(\noise_\timeindex)$ are independent, zero-mean finite-variance random variables, and $\succfunction$ denotes the probability that the learner gets a noisy informative response from the oracle.}

An eavesdropper observes query $\query_\timeindex$ and incentive $\price_\timeindex$ but not response $\response_\timeindex$. {The eavesdropper aims to estimate  $\learnerestimate$, as an approximation to the minimizer of the function the eavesdropper is interested in optimizing.} { This paper addresses the question: }{\textit{Suppose the learner uses a stochastic gradient (SG) algorithm to obtain an estimate $\learnerestimate$. How can the learner control the SG to hide $\learnerestimate$ from an eavesdropper?}}

{
Our proposed approach is to dynamically switch between two SGs. Let $\learnaction_\timeindex \in \{ 0 = \texttt{Obfuscate SG}, 1 = \texttt{Learn SG} \}$ denote the chosen SG at time $\timeindex$. The first SG  minimizes function $\function$ and updates the learner estimate $\learnerestimate_{\timeindex}$. The second SG is for obfuscating the eavesdropper with estimates $\obfuscatingestimate_\timeindex$. The update of both SGs is given by the equation,
\begin{align}\label{eq:controlledgradientstep}
\begin{bmatrix}
    \learnerestimate_{\timeindex+1}\\
    \obfuscatingestimate_{\timeindex+1}
\end{bmatrix}
     = \begin{bmatrix}
    \learnerestimate_{\timeindex}\\
    \obfuscatingestimate_{\timeindex}
\end{bmatrix} - \stepsize_\timeindex\begin{bmatrix}
\indicator(\learnaction_\timeindex=1) & 0 \\
0 & \indicator(\learnaction_\timeindex=0)
\end{bmatrix} 
\begin{bmatrix}
    \response_\timeindex\\
    \Bar{\response}_\timeindex
\end{bmatrix},
\end{align}
where $\stepsize_\timeindex$ is the step size, $\Bar{\response}_\timeindex$ is a synthetic gradient response discussed later and $\learnaction_\timeindex$ controls the SG to update.

The query $\query_\timeindex$ by the learner to the oracle is given by, 
\begin{align}\label{eq:query}
    \query_\timeindex &= \learnerestimate_\timeindex \indicator(\learnaction_\timeindex=1) + \obfuscatingestimate_\timeindex\indicator(\learnaction_\timeindex=0).
\end{align}
and $\action_\timeindex = (\learnaction_\timeindex,\price_\timeindex)$ is the control learner variable (action). 
The learner needs $\numsuccessfulupdates$ informative updates of~\eqref{eq:controlledgradientstep} to achieve the learning objective in $\numqueries$ queries. We formulate an MDP whose policy $\policy$ controls the switching of SGs and incentivization by the learner, to minimize the expected cost balancing obfuscation and learning. The optimal policy $\optpolicy$ solving the MDP is shown to have a threshold structure (Theorem~\ref{th:structural}) of the form, 
\begin{align*}
    \optpolicy(\learnerstate,\oraclestate,\dpindex) = \begin{cases}
       \learnaction = 0 \ (\texttt{obfuscate}), \ &  \learnerstate \leq \thresholdlower(\oraclestate,\dpindex) \\ 
      \learnaction = 1 \ (\texttt{learn}), \ & \learnerstate > \thresholdlower(\oraclestate,\dpindex)
    \end{cases},
\end{align*}
where $\learnerstate$ is the number of informative learning steps left, $\dpindex$ is the number of queries left and $\thresholdlower$ is the threshold function dependent on the oracle state $\oraclestate$ and $\dpindex$. Note that the exact dependence on the incentive is discussed later. We propose a stochastic approximation algorithm to estimate the optimal stationary policy with a threshold structure. We propose a multi-armed bandits based approach with finite-time regret bounds in Theorem~\ref{th:regret}. The optimal stationary policy with a threshold structure is benchmarked in a numerical study for covert federated hate-speech classification.
}

 {
\textbf{Motivation: } 
The main application of covert (or \textit{learner-private}) optimization is in centralized distributed optimization~\cite{xu_learner-private_2023,tsitsiklis_private_2021,shi_finite-time_2023}.  One motivating example is in pricing optimization and inventory management, the learner (e.g., e-retailer) queries the distributed oracle (e.g., customers) pricing and product preferences to estimate the optimal price and quantity of a product to optimize the profit function~\cite{tsitsiklis_private_2021,bersani_distributed_2019}. A competitor could spoof as a customer and use the optimal price and quantity for their competitive advantage. Our numerical experiment illustrates another application in federated learning, a form of distributed machine learning.

The current literature on covert optimization has been focused on deriving upper and lower bounds on the query complexity for a given obfuscation level~\cite{xu_learner-private_2023}. Query complexity for binary and convex covert optimization with a Bayesian eavesdropper has been studied in~\cite{tsitsiklis_private_2021,xu_learner-private_2023}. These bounds assume a static oracle and a random querying policy can be used to randomly obfuscate and learn. In contrast, the authors have looked at dynamic covert optimization where stochastic control is used to query a stochastic oracle optimally~\cite{jain_controlling_2024}. This is starkly different than the current literature since a stochastic oracle models situations where the quality of gradient responses may vary (e.g., due to Markovian client participation). The success of a response can be determined by the learner (e.g., based on gradient quality~\cite{jain_controlling_2024}) or by the oracle (e.g., based on computational resources availability). 

\textbf{Differences from previous work~\cite{jain_controlling_2024}}: To prove that the optimal policy has a monotone threshold structure,~\cite{jain_controlling_2024} requires supermodularity conditions. This paper proves results under more relaxed conditions using interval dominance~\cite{quah_comparative_2009} in Theorem~\ref{th:structural}, which can incorporate convex cost functions and more general transition probabilities~\cite{krishnamurthy_interval_2023}. The action space in this work includes an incentive the learner provides to the oracle.  An incentive that the learner pays is motivated by the learner's cost for obtaining a gradient evaluation of desired quality, it could be a monetary compensation the learner pays or non-monetary, e.g., controlling latency of services to participating clients ~\cite{witt_decentral_2023}. We had a generic cost function in~\cite{jain_controlling_2024}, but the costs considered in this paper are exact regarding the learner's approximation of the eavesdropper's estimate of $\learnerestimate$.
}
\vspace{-3mm}
\section{Covert optimization for first-order stochastic gradient descent}\label{sec:covertopt}

This section describes the two stochastic gradient algorithms, between which the learner dynamically switches to either learn or obfuscate using the MDP formulation of the next section. This section states the assumptions about the oracle, the learner, the eavesdropper, and the obfuscation strategy. We state the result on the number of successful gradient steps the learner needs to achieve the learning objective. The problem formulation for covert optimization is illustrated in Fig.~\ref{fig:sysmodel}.

\vspace{-4mm}
\subsection{Oracle}
The oracle evaluates the gradient of the function $\function$. The following is assumed about the oracle and the function $\function$, 
\begin{enumerate}[start=1,label={\bfseries O\arabic*:}]
    \item Function $\function:\R^\dimfunc \to \R$ is continously differentiable and is lower bounded by $\function^*$. Function $\function$ is $\lipschitz$-Lipschitz continuous, $\Vert \gradient \function(\functionvar) - \gradient\function(\functionvaralt)\Vert \leq \lipschitz \Vert \functionvar - \functionvaralt \Vert \ \forall \functionvar,\functionvaralt \in \R^\dimfunc$.
    \item At time $\timeindex$, the oracle is in state $\oraclestate_\timeindex \in \{1,\dots,\oraclelevels\}$, { where $\oraclelevels$ are the number of oracle states} and for the incentive $\price_\timeindex \in \{ \price^1,\dots, \price^\numprices \}$, replies {with probability} $\succfunction(\oraclestate_\timeindex,\price_\timeindex)$. $\succreply_\timeindex \sim
    \text{Bernoulli}(\succfunction(\oraclestate_\timeindex,\price_\timeindex))$ 
    denotes success of the reply.
    \item For a reply with success $\succreply_\timeindex$ to the query $\query_\timeindex \in \R^\dimfunc$, the oracle returns a noisy gradient response $\response_\timeindex$ according to~\eqref{eq:oraclereply}. The noise terms $\noise_\timeindex$ are independent~\footnote{A slightly weaker assumption based on conditional independence was considered in our paper~\cite{jain_controlling_2024}. We consider independence here for brevity.}, have zero-mean and finite-variance, $\expectation[\response_\timeindex ] = \gradient\function(\query_\timeindex)$ and $\expectation[\Vert\noise_\timeindex\Vert^2] \leq \noisevariance$. 
\end{enumerate}
O1 and O3 are standard assumptions for analyzing oracle-based gradient descent~\cite{ghadimi_stochastic_2013}. O2 is motivated by an oracle with a stochastic state (e.g., client participation), and the success is determined by the oracle or by the learner.  
\begin{figure}
    \centering
    \vspace{1mm}
    \include{figure}
    \vspace{-6mm}
    \caption{ Dynamic Covert Optimization: { Learner sends query $\query_\timeindex$ and incentive $\price_\timeindex$ to oracle in state $\oraclestate_\timeindex$. The oracle evaluates noisy gradient of $\function$ at $\query_\timeindex$, $\response_\timeindex$ according to~\eqref{eq:oraclereply}. An eavesdropper observes $\query_\timeindex$ and $\price_\timeindex$ and aims to approximate the learner's estimate. The learner needs to control the incentive $\price_\timeindex$ and type of SG ($\learnaction_\timeindex$) to query using~\eqref{eq:query}  to achieve the learning objective of~\eqref{eq:learnerobjective} and obfuscate the eavesdropper with belief~\eqref{eq:eavesestimate}.}}
    \label{fig:sysmodel}
    \vspace{-7mm}
\end{figure}
\vspace{-2mm}
\subsection{Learner}

Similar to oracle-based first-order gradient descent~\cite{ghadimi_stochastic_2013}, the learner aims to estimate $\learnerestimate\in \R^\dimfunc$ which is a $\epsilon$-close critical point of the function $\function$,
\begin{align}\label{eq:learnerobjective}
 \expectation\left[ \Vert \nabla \function( \learnerestimate)\Vert^2 \right]\leq \epsilon.
\end{align}
Since $\function$ is non-convex and not known in closed-form to the learner, in general, the gradient at $\functionvaralt_1$ is non-informative about the gradient at $\functionvaralt_2$ far from $\functionvaralt_1$. Hence, at time $\timeindex$, the learner can either send a learning or an obfuscating query. 
We propose controlling the gradient descent of the learner by the query action $\learnaction_\timeindex \in \{ 0 = \texttt{obfuscating}, 1 = \texttt{learning}\}$. While learning, the learner updates its estimate, $\learnerestimate_\timeindex$ by performing the controlled stochastic gradient step of~\eqref{eq:controlledgradientstep}. 
Here, $\stepsize_\timeindex$ is the step size chosen to be constant in this paper. 
In the next section, we will formally state the action space composed of the type of query $\learnaction_\timeindex$ and the incentive $\price_\timeindex$. In order to estimate the number of queries to the oracle that the learner has to spend on learning queries, we first define the successful gradient step. We then state the result on the order of the number of successful gradient steps required for achieving the objective. 
\begin{definition}~\label{def:succgradientstep}\textbf{(Successful Gradient Step) } A gradient step of \eqref{eq:controlledgradientstep} is successful when the learner queries the oracle with a learning query ($\learnaction_\timeindex = 1$) and gets a successful reply ($\succreply_\timeindex = 1$). 
\end{definition}
\begin{theorem}\label{th:numsuccupdates}
For an oracle with assumptions (O1-O3), to obtain an estimate $\learnerestimate$ which achieves the objective~\eqref{eq:learnerobjective}, the learner needs to perform $\numsuccessfulupdates$ successful gradient steps (Def.~\ref{def:succgradientstep}) with a step size ($\stepsize = \min(\frac{1}{\lipschitz},\frac{\epsilon}{2\sigma^2\lipschitz})$) where $\numsuccessfulupdates$ is of the order, $O\left(\frac{\noisevariance}{\epsilon^2} + \frac{1}{\epsilon}\right).$ {The exact expression is $M = \max\left({\frac{4F\lipschitz}{\epsilon},\frac{8F\lipschitz\noisevariance}{\epsilon^2}}\right)$
where $F =(\expectation \function (\functionvar_0) - \function^*)$.}
\end{theorem}
\textit{Proof for a general setting can be found in~\cite{jain_controlling_2024} and in~\cite{ghadimi_stochastic_2013}. } Theorem~\ref{th:numsuccupdates} characterizes the number of successful gradient steps, $\numsuccessfulupdates$ that the learner needs to perform to achieve the learning objective of~\eqref{eq:learnerobjective}. {Theorem~\ref{th:numsuccupdates} guarantees the existence of a finite queue state space in the next section, which models the number of successful gradient steps left to be taken. $\numsuccessfulupdates$ in the MDP formulation of the next section can be chosen heuristically or be computed exactly if the parameters of the function are known. It also shows that $\numsuccessfulupdates$ is inversely dependent on $\epsilon$ and incorporates the descent dynamics in the structure of the optimal policy. }  
\subsection{Obfuscation Strategy}
{Based on the chosen SG $\learnaction_\timeindex$, the learner poses queries using~\eqref{eq:query}} and provides incentives to the oracle. To obfuscate the eavesdropper, the learner runs a parallel stochastic gradient with synthetic responses, $\Bar{\response}_\timeindex$. {The synthetic responses can be generated by suitably simulating an oracle, for e.g., the learner can train a neural network separately with an unbalanced subset of as was done in~\cite{jain_controlling_2024}. If the learner is sure that the eavesdropper has no public dataset to validate, the learner can simply take mirrored gradients with \eqref{eq:oraclereply}. When obfuscating, the learner poses queries from the estimates of the second SG, $\obfuscatingestimate_\timeindex$.}

The parallel stochastic gradient ensures that the eavesdropper cannot infer the true learning trajectory from the shape of the trajectory. { In summary, the learner obfuscates and learns by dynamically chooses the query $\query_\timeindex$, as the current estimate $\learnerestimate_{\timeindex}$ from the controlled stochastic gradient step or as the estimate $\obfuscatingestimate_{\timeindex}$ of parallel SG.} We assume that the learner queries such that the two trajectories are sufficiently separated from each other, and the eavesdropper can cluster the queries and distinguish them uniquely into two trajectories as described next.
\vspace{-6mm}
\subsection{Eavesdropper}
At time $\timeindex$, the eavesdropper observes query $\query_\timeindex$ and the incentive $\price_\timeindex$ by the learner. We use the following assumptions to approximate the posterior belief of the eavesdropper,
\begin{enumerate}[start=1,label={\bfseries E\arabic*:}]
    \item The eavesdropper is passive, does not affect the oracle, and can not observe the oracle's responses, $\response_\timeindex$. 
    \item For $\timeindex>0$, the eavesdropper can classify each query of the observed query sequence $(\query_1,\dots,\query_\timeindex)$ into two unique trajectory sequences, $\truetrajectory^\timeindex$ and $\falsetrajectory^\timeindex$ which can be separated by a hyperplane in $\R^\dimfunc$. 
    \item The eavesdropper knows that either of $\truetrajectory^\timeindex$ and $\falsetrajectory^\timeindex$ is a learning trajectory. Given the information till time $\timeindex$, the eavesdropper computes a posterior belief, $\trajectoryprobability_\timeindex$ for trajectory $\truetrajectory^\timeindex$ being the learning trajectory as, 
    \begin{align*}
        \trajectoryprobability_\timeindex = \probability(\learnerestimate \in \truetrajectory^\timeindex|(\query_1,\price_1,\dots,\query_\timeindex,\price_\timeindex)).
    \end{align*}
\end{enumerate}
Assumption E1 holds if the eavesdropper is considered an insignificant part of the oracle and can not observe the aggregated response.
Assumptions E2 and E3 ensure the dynamic obfuscation is for a worst-case eavesdropper{~\footnote{As mentioned above, it is assumed that the queries are posed such that the two trajectories are sufficiently separated (by a metric known to the eavesdropper). One of the trajectories can be empty for the initial queries.}}.
We consider the eavesdropper using an incentive-weighed proportional sampling estimator for the posterior $\trajectoryprobability_\timeindex$~\cite{xu_learner-private_2023,jain_controlling_2024}, 
\begin{align}\label{eq:eavesestimate}
    \trajectoryprobability_\timeindex  = \frac{\sum_{1 \leq \timeindexalt \leq \timeindex} \price_\timeindexalt \indicator(\query_\timeindexalt \in \truetrajectory^\timeindexalt)}{\sum_{1 \leq \timeindexalt \leq \timeindex} \price_\timeindexalt}.
\end{align}
The learner has access to the queries and the incentives and uses~\eqref{eq:eavesestimate} as an approximation for the eavesdropper's $\argmin$ estimate to compute the cost incurred when learning. The following section formulates an MDP to perform covert optimization using stochastic control. $\numsuccessfulupdates$ from Theorem~\ref{th:numsuccupdates} and oracle state are used to model the state space, while the incentives $\price_\timeindex$ and {the type of SG $\learnaction_\timeindex$ in~\eqref{eq:controlledgradientstep}} model the action space.
\section{{MDP} for achieving covert optimization}\label{sec:mdp}

We formulate a finite-horizon MDP to solve the learner's decision problem. The learner chooses an incentive, and dynamically either minimizes the function using the estimate $\learnerestimate_\timeindex$ or obfuscates the eavesdropper using $\obfuscatingestimate_\timeindex$. The learner wants to perform $\numsuccessfulupdates$ successful gradient steps in $\numqueries$ total queries. Using interval dominance, we show that the optimal policy of the finite-horizon MDP has a threshold structure. {The stochastic control approach for the same is described in Algorithm~\ref{alg:stochasticcontrol}.}
\vspace{-3mm}
\subsection{MDP formulation for optimally switching between stochastic gradient algorithms}\label{sec:mdpform}
\begin{algorithm}[h!]

    \begin{algorithmic}
    \State \textbf{Input: } Policy $\policy$, Queries $\numqueries$, Successful Gradient Steps $\numsuccessfulupdates$
    \State Initialize learner queue state $\learnerstate_\numqueries = \numsuccessfulupdates$
\For{$\timeindex$ in $1,\dots,\numqueries$}
        \State Obtain type of SG and incentive, $(\learnaction_\timeindex,\price_\timeindex) = \policy(\oraclestate_\timeindex,\learnerstate_\timeindex)$
        \State Incur cost $\cost((\learnaction_\timeindex,\price_\timeindex),(\oraclestate_\timeindex,\learnerstate_\timeindex))$ from~\eqref{eq:learningcost}
        \State Query oracle using query $\query_\timeindex$~\eqref{eq:query} and incentive $\price_\timeindex$
        \State Receive response $\response_\timeindex$ and success of reply $\succreply_\timeindex$
        \State Update estimates of the two SGs using~\eqref{eq:controlledgradientstep}.
        \State \textbf{if }{$\succreply_\timeindex$=1} \textbf{then }$\learnerstate_{\timeindex+1} = \learnerstate_{\timeindex} - \succreply_\timeindex$
        \State Oracle state evolves, $\oraclestate_{\timeindex+1} \sim \oracleevolutionprob{\oraclestate_\timeindex}{\cdot}$
        \EndFor{}
        \State Incur terminal cost $\queuecost(\learnerstate_0)$
    \end{algorithmic}
    \caption{Stochastic Control for Covert Optimization}
    \label{alg:stochasticcontrol}
\end{algorithm}
\vspace{-1mm}
The dynamic programming index, $\dpindex = \numqueries, \dots, 0$ denotes the number of queries left and decreases with time $\timeindex$.

\textbf{State Space: }
The state space is denoted by $\statespace = \statespace^\learnersymbol \times \statespace^\oraclesymbol$ where $\statespace^\learnersymbol = \{ 0, 1, \dots , \numsuccessfulupdates \}$ is the learner queue state space and $\statespace^\oraclesymbol = \{ 0, 1, \dots, \oraclelevels \}$ is the oracle state space. The learner queue state $\learnerstate_\dpindex \in \statespace^\learnersymbol$ denotes the number of successful gradient steps (Def.~\ref{def:succgradientstep}) remaining to achieve~\eqref{eq:learnerobjective}.  The oracle state space $\statespace^\oraclesymbol$ discretizes the stochastic state of the oracle into $\oraclelevels$ levels (e.g., percentages of client participation in FL). $\statevar_\dpindex$ denotes the state with $\dpindex$ queries remaining.

\textbf{Action Space: }
The action space is $\actionspace = \{0 = \texttt{obfuscate}, 1 = \texttt{learn} \}\times \{\price^1,\dots,\price^\numprices\}$. The action when $\dpindex$ queries are remaining is given by, $\action_\dpindex = (\learnaction_\dpindex,\price_\dpindex)$ where $\learnaction_\dpindex \in \{0 = \texttt{obfuscate}, 1 = \texttt{learn} \}$ is the type of the query and $\price_\dpindex \in \{\price^1,\dots,\price^\numprices\}$ is the incentive. To derive structural results on the optimal policy, we consider the following transformation of the action space, $\actionspace = \{ (0,\price^1), \dots, (0,\price^\numprices), (1,\price^1), \dots, (1,\price^\numprices)  \}$. A deterministic policy for the finite-horizon MDP is denoted by $\policy$, a sequence of functions $\policy = (\action_\dpindex: \dpindex=0,\dots,\numqueries)$. Here, $\action_\dpindex:\statespace\to\actionspace$ maps the state space to the action space. $\policyspace$ denotes the space of all policies.

\textbf{Transition Probabilities: }
We assume that the evolution of the oracle state and the learner queue state is Markovian. The oracle state evolves independently of the queue state evolution. In case of a successful gradient step (Def.~\ref{def:succgradientstep}), the queue decreases by one, and the oracle state evolves to a state $\oraclestate^\prime \in \statespace^\oraclesymbol$ with probability $\oracleevolutionprob{\oraclestate}{\oraclestate^\prime} {> 0}$.  {Let $\transitionqueue{}(\action) = \nicefrac{\probability((\oraclestate^\prime,\cdot)|\oraclestate,\learnerstate,\action)}{\oracleevolutionprob{\oraclestate}{\oraclestate^\prime}}$ denote the transition probability vector of the buffer state with future oracle state $\oraclestate^\prime$ given $(\oraclestate,\learnerstate,\action)$.} The transition probability from the state $\statevar = (\oraclestate,\learnerstate) \in \statespace$ to state $\statevar^\prime = (\oraclestate^\prime,\learnerstate^\prime) \in \statespace$ with action $\action = (\learnaction,\price)$ can be written as,
\begin{align}\label{eq:transitionprobability}
\begin{split}
    &\probability(\statevar^\prime =(\oraclestate^\prime,\learnerstate-1)|\statevar,\action) = \oracleevolutionprob{\oraclestate}{\oraclestate^\prime}\succfunction(\oraclestate,\price)\indicator(\learnaction) \ \forall \oraclestate^\prime,
    \\
    &\probability(\statevar^\prime = (\oraclestate,\learnerstate)|\statevar,\action) = \left(1 -\succfunction(\oraclestate,\price)\right)\indicator(\learnaction) + \left(1 - \indicator(\learnaction)\right),
\end{split}
\end{align}
and is $0$ otherwise.
The first equation corresponds to a successful gradient step, and the second to an unsuccessful one. We assume that $\succfunction(\oraclestate,\price)$ (from O2) is increasing in incentive $\price$.

\textbf{Learning and Queueing Cost: } The learning cost $\cost_\dpindex: \statespace\times\actionspace \to \R$, is the cost (with $\dpindex$ queries remaining) incurred after every action due to learning at the expense of reduced obfuscation. We consider the following learning cost which is proportional to the logarithm of the improvement in the eavesdropper's estimate ($\propto\log(\nicefrac{\trajectoryprobability_\dpindex}{\trajectoryprobability_{\dpindex+1}})$) and is given by,  
\begin{align}\label{eq:learningcost}
    \begin{split}
        \cost_\dpindex(\statevar_\dpindex,\action_\dpindex) &= \frac{\buffercost(\learnerstate_\dpindex)}{{\oraclecost(\oraclestate_\dpindex)}}  \log \left(\frac{\previousincentivesum + \nicefrac{\price_\dpindex}{\trajectoryprobability_\dpindex}}{\previousincentivesum + \price_\dpindex} \right)\indicator(\learnaction_\dpindex) \\
             + &\frac{ \oraclecost(\oraclestate_\dpindex)}{\buffercost(\learnerstate_\dpindex)} \log \left(\frac{\previousincentivesum}{\previousincentivesum + \price_\dpindex} \right)(1-\indicator(\learnaction_\dpindex))
    \end{split},
\end{align}
{where $\buffercost:\statespace^\learnersymbol\to\R^{+}$ and $\oraclecost:\statespace^\oraclesymbol\to\R^{+}$ are positive, convex and increasing cost functions, $\previousincentivesum = \sum_{\timeindex=\numqueries-1}^{\dpindex+1} \price_\timeindex$ is the sum of the previous incentives and $\trajectoryprobability_\dpindex$ is the eavesdropper's estimate of the trajectory $\truetrajectory$ being the true trajectory computed using (\ref{eq:eavesestimate}). $\buffercost$ and $\oraclecost$ are used to incorporate the cost with respect to the oracle and queue state, e.g., the functions $\buffercost$, and $\oraclecost$ are considered quadratic in the respective states in the experiments. The form of the fractions ensures the structure as discussed next.}
The first term in~\eqref{eq:learningcost} denotes the cost incurred in a learning query and is non-negative ($ 0\leq\trajectoryprobability_\timeindex\leq1$). The second term corresponds to an obfuscating query and is non-positive. The cost increases with the queue state and decreases with the oracle state. This incentivizes the learner to drive the system to a smaller queue and learn when the oracle is in a good state. {After $\numqueries$ queries, the learner pays a terminal queue cost computed using the function $\queuecost:\statespace\to\R$.} The queue cost accounts for learning loss in terms of terminal successful gradient steps left, $\learnerstate_0$. 

\textit{Remark: } The incentive improves the response probability $\succfunction$, but also allows for improved obfuscation than a non-incentivized setup (a high incentive can be used to misdirect the eavesdropper's belief in~\eqref{eq:eavesestimate}). 
\vspace{-4mm}
\subsection{Optimization problem}
The expected total cost for the finite-horizon MDP with the initial state $\statevar_\numqueries \in \statespace$ and policy $\policy$ is given by, 
\begin{align}\label{eq:cummalativecost}
\valuefn^\policy(\statevar_\numqueries) =  \expectation\left[ \frac{1}{\numqueries}\sum_{\dpindex=1}^{\numqueries} \cost_\dpindex\left(\statevar_\dpindex,\action_\dpindex\right) + \queuecost(\statevar_0,\action_0)\mid\statevar_\numqueries,\policy\right].
\end{align}
The optimization problem is to find the optimal policy $\optpolicy$, 
\begin{align}\label{eq:mdp}
    \valuefn^{\optpolicy}(\statevar) = \inf_{\policy \in \policyspace} \valuefn^{\policy}(\statevar) \ \forall \ \statevar\in\statespace.
\end{align}
To define the optimal policy using a recursive equation, we first define the value function, $\valuefn_\dpindex$ with $\dpindex$ queries remaining, 
\begin{align}\label{eq:valuefn}
    \valuefn_\dpindex(\statevar) = \underset{\action \in \actionspace}{\min} \left(\ \cost_\dpindex(\statevar,\action) + \sum_{\statevar^\prime \in \statespace}\probability(\statevar^\prime|\statevar)\valuefn_{\dpindex-1}(\statevar^\prime)\right).
\end{align}
Let the optimal policy be $\optpolicy = (\optaction_\dpindex)_{\dpindex=N}^{1}$, where $\optaction(\cdot)$ is the optimal action with $\dpindex$ remaining queries and is the solution of the following stochastic recursion (Bellman's equation),
\begin{align}\label{eq:optpolicy}
\optaction_\dpindex(\statevar) = \underset{\action \in \actionspace}{\argmin} \  \qfunction_\dpindex(\action,\statevar),
\end{align}
where the Q-function $\qfunction_\dpindex$ is defined as,
\begin{align}\label{eq:qfunc}
\qfunction_{\dpindex}(\action,\statevar) = \cost_\dpindex(\action,\statevar) + \sum_{\statevar^\prime \in \statespace} \probability(\statevar^\prime|\statevar,\action)\valuefn_{\dpindex - 1}(\statevar^\prime),
\end{align}
with $\dpindex = 0, \dots, \numqueries$ and $\valuefn_0(\statevar) = \queuecost(\statevar)$.
If the transition probabilities are unknown, then  Q-learning can be used to estimate the optimal policy of~\eqref{eq:optpolicy}. However, the following subsection shows that the optimal policy has a threshold structure, which motivates efficient policy search algorithms. 
\vspace{-6mm}
\subsection{Structural Results}
\vspace{-1mm}
The following is assumed to derive the structural results,
\begin{enumerate}[start=1,label={\bfseries R\arabic*:}]
\item The learning cost, $\cost_\dpindex$ is $\increasing$ (increasing) and convex in the buffer state, $\queuecost_\dpindex$ for each action $\action_\dpindex \in \actionspace$. 
\item Transition probability matrix $\probability(\learnerstate^\prime|\learnerstate,\oraclestate,\action)$ is TP3\footnote{Totally postive of order 3 (TP3) for a matrix P(a) requires that each of 3rd order minor of P(a) is non-negative.} with $\sum_{\learnerstate^\prime}\learnerstate^\prime\probability(\learnerstate^\prime|\learnerstate,\oraclestate,\action) \increasing \learnerstate$ and convex in $\learnerstate$.
\item The terminal cost, $\queuecost$ is $\increasing$ and convex in the queue state, $\learnerstate$.
\item For $\idfactorone_{\learnerstate^{\prime},\learnerstate,\action}>0$ and $\increasing \action$, $\cost(\learnerstate^{\prime},\oraclestate,\action+1)-\cost(\learnerstate^{\prime},\oraclestate,\action)\leq \idfactorone_{\learnerstate^{\prime},\learnerstate,\action}(\cost(\learnerstate,\oraclestate,\action+1)-\cost(\learnerstate,\oraclestate,\action))$, $\learnerstate^{\prime}> \learnerstate$.
\item For $\idfactortwo_{\learnerstate^{\prime},\learnerstate,\action}>0$ and $\increasing \action$, 

$\frac{\transitionqueue{\prime}(\action + 1)+ \idfactortwo_{\learnerstate^{\prime},\learnerstate,\action}\transitionqueue{\prime}(\action) }{1 + \idfactortwo_{\learnerstate^{\prime},\learnerstate,\action}} <_{c}\frac{\transitionqueue{}(\action)+ \idfactortwo_{\learnerstate^{\prime},\learnerstate,\action}\transitionqueue{}(\action+1) }{1 + \idfactortwo_{\learnerstate^{\prime},\learnerstate,\action}}$

$,\learnerstate^{\prime}>\learnerstate, \forall \oraclestate^\prime,\oraclestate \in \statespace^\oraclesymbol$; $<_c$ denotes convex dominance~\footnote{\label{fn:cvxdom}Probability vector $p$ is convex dominated by probability vector $q$ iff $f'p \geq f'q$ for increasing and convex vector $f$.}.
\item There exist $\idfactorone_{\learnerstate^{\prime},\learnerstate,\action}= \idfactortwo_{\learnerstate^{\prime},\learnerstate,\action}$ s.t. (R4) and (R5) hold.
\end{enumerate}
Assumptions (R1) and (R3) are true by the construction of cost in~\eqref{eq:learningcost} and the terminal cost. (R2) is a standard assumption on bi-diagonal stochastic matrices made when analyzing structural results~\cite{krishnamurthy_interval_2023}. (R4), (R5) and (R6) are the generalization of the supermodularity conditions made previously in~\cite{jain_controlling_2024} and are sufficient for interval dominance~\cite{krishnamurthy_interval_2023}. Assumption (R4) can be verified for cost of~\eqref{eq:learningcost} using algebraic manipulation with $\idfactorone_{\learnerstate^{\prime},\learnerstate,\action} \leq 1$, and (R5) can be shown with $\idfactortwo_{\learnerstate^{\prime},\learnerstate,\action} \leq 1$ for the bi-diagonal matrix of ~\eqref{eq:transitionprobability} with $\succfunction(\oraclestate,\price) \increasing \price$. Therefore (R6) can be satisfied for some $\idfactorthree_{\learnerstate^{\prime},\learnerstate,\action} = \idfactorone_{\learnerstate^{\prime},\learnerstate,\action}= \idfactortwo_{\learnerstate^{\prime},\learnerstate,\action} \leq 1$. We now state the main structural result,
\begin{theorem}\label{th:structural}
    Under assumptions (R1-6), the optimal action $\optaction_\dpindex(\statevar)$ (given by \eqref{eq:optpolicy}) for the finite-horizon MDP of \eqref{eq:mdp} is increasing in the queue state $\learnerstate$.
\end{theorem}
\begin{proof}
    \textbf{Step 1: Conditions of Interval Dominance: }
    
    The following condition with $\idfactorthree_{\learnerstate^\prime,\learnerstate,\action}>0$,
\begin{align}\label{eq:idcondition}
\begin{split}
    \qfunction_\dpindex(\learnerstate^\prime,\action+1) &- \qfunction_\dpindex(\learnerstate^\prime,\action) \leq \\&\idfactorthree_{\learnerstate^\prime,\learnerstate,\action}\left[    \qfunction_\dpindex(\learnerstate,\action+1) - \qfunction_\dpindex(\learnerstate,\action) \right], \learnerstate^\prime>\learnerstate,
    \end{split}
    \end{align}
    is sufficient for $\arg\min \qfunction_\dpindex$ to be increasing in $\learnerstate$ ~\cite{quah_comparative_2009,krishnamurthy_interval_2023},
        \vspace{-2mm}
    \begin{align*}
        \optaction_\dpindex(\learnerstate) = \argmin_{\action \in \actionspace} \qfunction_\dpindex(\learnerstate,\action) \increasing \learnerstate.
    \end{align*}

    We omit the oracle state $\oraclestate$ from the above expression. 
    
    By plugging~\eqref{eq:qfunc} in~\eqref{eq:idcondition} we need to show the following, 
    \begin{align*}
    \begin{split}
&\underbrace{\cost_\dpindex(\learnerstate^\prime,\action+1) - \cost_\dpindex(\learnerstate^\prime,\action) - \idfactorthree_{\learnerstate,\learnerstate^\prime,\action}(\cost_\dpindex(\learnerstate,\action+1) - \cost_\dpindex 
        (\learnerstate,\action))}_{a} \\&+ \sum_{\oraclestate^\prime \in \statespace^\oraclesymbol } \sum_{\learnerstate^{''}}\left[\probability((\oraclestate^\prime,\learnerstate^{''})|(\oraclestate,\learnerstate^\prime),\action+1)  \right.\\&\left.-\probability((\oraclestate^\prime,\learnerstate^{''})|(\oraclestate,\learnerstate^\prime),\action)
        - \idfactorthree_{\learnerstate,\learnerstate^\prime,\action}\left(
    \probability((\oraclestate^\prime,\learnerstate^{''})|(\oraclestate,\learnerstate),\action+1) \right.\right.\\&\left.\left.- \probability((\oraclestate^\prime,\learnerstate^{''})|(\oraclestate,\learnerstate),\action) 
        \right)
\right]\valuefn(\oraclestate^\prime,\learnerstate^{''}) \leq 0, \learnerstate^\prime>\learnerstate
\end{split}.
    \end{align*}
    \vspace{-0.5mm}
By (R4) part (a) of the above inequality is satisfied with constant $\idfactorone_{\learnerstate,\learnerstate^\prime,\action} \leq 1$. The rest of the inequality can be shown using (R5) with a constant $\idfactortwo_{\learnerstate,\learnerstate^\prime,\action} \leq 1$ if we assume the value function is increasing and convex (see n.\ref{fn:cvxdom} and ~\cite{krishnamurthy_interval_2023}). Finally we apply (R6), with $\idfactorone_{\learnerstate,\learnerstate^\prime,\action} = \idfactortwo_{\learnerstate,\learnerstate^\prime,\action} = \idfactorthree_{\learnerstate,\learnerstate^\prime,\action}$ to show that \eqref{eq:idcondition} holds and the optimal action is $\increasing$ in learner state $\learnerstate$.   All that remains to be shown is that the value function is increasing and convex, which we now show using (R1, R2, R3) and induction,

\textbf{Step 2: Value Function  is Increasing in $\learnerstate$: }
By (R3), $\valuefn_0(\statevar) = \queuecost(\learnerstate)$ is increasing in $\learnerstate$. Let $\valuefn_{\dpindex}(\statevar) \increasing \learnerstate$. TP3 (R2) implies TP2 and hence preserves monotone functions~\cite{krishnamurthy_interval_2023}. Therefore by applying preservation of TP2 and linear combination, $\sum_{\oraclestate^\prime \in \statespace^\oraclesymbol }\oracleevolutionprob{\oraclestate}{\oraclestate^{'}}\sum_{\learnerstate^{'}\in\statespace^\learnersymbol}\probability(\learnerstate^{'}|\learnerstate,\oraclestate,\action) \valuefn_\dpindex \increasing \learnerstate$. By (R1) and (\ref{eq:qfunc}), $\qfunction_{\dpindex+1} \increasing \learnerstate$. And therefore by \eqref{eq:valuefn}, $\valuefn_{\dpindex+1}(\statevar) \increasing \learnerstate$.  

\textbf{Step 3: Value Function is Convex in $\learnerstate$: }
    By (R3) $\valuefn_0(\statevar) = \queuecost(\learnerstate)$ is convex in $\learnerstate$. Let $\valuefn_\dpindex$ be convex in $\learnerstate$. Then by (R2) and applying Lemma 1 of~\cite{krishnamurthy_interval_2023} along with preservation of convexity under positive weighted sum, $\sum_{\oraclestate^\prime \in \statespace^\oraclesymbol }\oracleevolutionprob{\oraclestate}{\oraclestate^{'}}\sum_{\learnerstate^{'}\in\statespace^\learnersymbol}\probability(\learnerstate^{'}|\learnerstate,\oraclestate,\action) \valuefn_\dpindex$ is convex in $\learnerstate$. Applying (R1) and (\ref{eq:qfunc}), $\qfunction_{\dpindex+1}$ is convex in $\learnerstate$. Since minimization preserves convexity, $\valuefn_{\dpindex+1} = \min \qfunction_{\dpindex+1}$ is convex in $\learnerstate$.
\end{proof}
{Theorem~\ref{th:structural} implies that the policy is threshold in the learner queue state; hence, the learner learns more aggressively when the number of successful gradient steps (Def.~\ref{def:succgradientstep}) left is more. This intuitively makes sense from an obfuscation perspective since the learner should ideally spend more time obfuscating when it is closer to the minimizer (the queue state is small).} 

Using Theorem~\ref{th:structural}, we can parameterize the optimal policy by the thresholds on the queue state. Although we can construct stochastic approximation for estimating the non-stationary policy, which has a threshold structure and performs computationally better than Q-learning, this approach still requires the number of parameters to be linear in time horizon $\numqueries$. Given this insight, we restrict the search space to stationary policies with a monotone threshold structure, this restriction is common in literature~\cite{jain_controlling_2024,ngo_monotonicity_2010}. 

Let the threshold on queue state for oracle state $\oraclestate$ and action $\action$ be parameterized by $\thresholdlower: \statespace^\oraclesymbol \times \actionspace\to \statespace^\learnersymbol$. The optimal stationary policy with a threshold structure can be written as,
\begin{align}\label{eq:optthresholdpolicy}
    \optpolicy(\statevar) = \sum_{\action \in \actionspace} \action\indicator(\thresholdlower^*(\oraclestate,\action)\leq \learnerstate < \thresholdlower^*(\oraclestate,\action+1)),
\end{align}
where $\thresholdlower^*$ is the optimal threshold function.
\vspace{-1mm}
\section{Estimating the optimal stationary policy with a threshold structure}\label{sec:algs}
In this section, we propose two methods to approximate the optimal stationary policy\footnote{In this section, the optimal stationary policy is referred to as optimal policy.}. for the finite-horizon MDP of~\eqref{eq:mdp} which has the monotone threshold structure of \eqref{eq:optthresholdpolicy}. The first method uses a stochastic approximation to update the parameters over the learning episodes iteratively. The second method uses a multi-armed bandit formulation to perform discrete optimization over the space of thresholds. The proposed methods can be extended to a non-stationary policy space with an increased time and memory complexity. 
\vspace{-4mm}
\subsection{Simultaneous Perturbation Stochastic Approximation}

Taking the thresholds of the stationary policy of ~\eqref{eq:optthresholdpolicy} as the parameters, a simultaneous perturbation stochastic approximation (SPSA) based algorithm can be used to find the parameters for the optimal policy. We update the policy parameters using approximate gradients of the costs computed using perturbed parameters. We use the following sigmoidal approximation for the threshold policy of~\eqref{eq:optthresholdpolicy},
\begin{align}\label{eq:sigmoidalapprox}
    \approxpolicy(\statevar,\thresholdlower) = \sum_{\action \in \actionspace} \frac{1}{1 + \exp\left(\nicefrac{-(\learnerstate - \thresholdlower(\oraclestate,\action))}{\sigmoidalparameter}\right)},
\end{align}
where $\sigmoidalparameter$ is an approximation parameter. The parameters are the $\thresholdcount =|\actionspace||\statespace_\oraclesymbol|$ threshold values and are represented by $\parametersspsa$.
For the optimal parameters, the approximate policy converges to the optimal policy as $\sigmoidalparameter \to 0$~\cite{ngo_monotonicity_2010}. For the learning episode $\spsatimeindex$ and current parameter set $\parametersspsa_\spsatimeindex$, the actions are computed using the current approximate policy~\eqref{eq:sigmoidalapprox}. The policy parameters are perturbed independently with probability~\nicefrac{1}{2} by $\pm\perturbation$. Two learning episodes are performed with each set of perturbed policy parameters ($\parametersspsa_\spsatimeindex^+$,$\parametersspsa_\spsatimeindex^-$). The costs from the two episodes are used to obtain the approximate gradient $\spsaapproxgradient_\spsatimeindex$ by the method of finite differences. The policy parameters are updated using a stepsize $\spsastepsize_\spsatimeindex$,
\begin{align*}
\parametersspsa_{\spsatimeindex+1} = \parametersspsa_\spsatimeindex - \spsastepsize_\spsatimeindex \spsaapproxgradient_\spsatimeindex.
\end{align*}
Under regularity conditions on the noise in the approximate gradients, the approximate policy parameters asymptotically converge in distribution to the set of parameters of the optimal stationary policies with a threshold structure~\cite{kushner_stochastic_2003}. 
The SPSA algorithm can also be used with a constant step size to track changes in the system~\cite{jain_controlling_2024,kushner_stochastic_2003}.
The computational complexity for each learning episode is $O(\thresholdcount+\numqueries)$.
  \vspace{-5mm}
\subsection{Multi-armed bandit approach}
The problem of searching the thresholds for (\ref{eq:optthresholdpolicy}) is solved by considering the values each threshold can take and then taking the product space of the thresholds as bandit arms. Each threshold can take values over learner state space $\statespace_\learnersymbol$, which is of the cardinality $\numsuccessfulupdates+1$.
Consider each permutation of the $\thresholdcount = |\actionspace||\statespace_\oraclesymbol|$ thresholds as an arm, making the total number of arms $(\numsuccessfulupdates+1)^\thresholdcount$. The selection of an arm gives a corresponding stationary policy of the form~\eqref{eq:optpolicy}, and a reward (negative of the cumulative cost of the episode) is obtained by interacting with oracle for time horizon $\numqueries$. The noisy reward is sampled from a distribution centered at the expected value~\eqref{eq:cummalativecost}. \textbf{(B1)} The reward is assumed to be sampled independently for a given policy, and the noise is assumed to be sub-Gaussian~\cite{bubeck_regret_2012}. {For brevity, we omit the definition of regret and the exact upper bound, both of which can be found in Chapter 2 of ~\cite{bubeck_regret_2012}. }Following is the result for the upper bound on the regret for searching the thresholds,
\begin{theorem}\label{th:regret}
    Consider the finite-horizon MDP of (\ref{eq:mdp}) for covert optimization with an oracle (O1-O3) to achieve \eqref{eq:learnerobjective}. The optimal stationary policy with a threshold structure~\eqref{eq:optthresholdpolicy} can be searched using the upper confidence bound algorithm under (B1) with an expected regret after $\numinteractions$ episodes bounded by $O\left(\numsuccessfulupdates^{\thresholdcount}\log{\numinteractions}\right)$, where, $\numsuccessfulupdates$ is of the order $ O(\nicefrac{1}{\epsilon} + \nicefrac{\noisevariance}{\epsilon^2})$ and $\thresholdcount = |\statespace_\oraclesymbol||\actionspace|$ is the number of thresholds.
\end{theorem}
 The proof follows from Theorem~\ref{th:numsuccupdates} and plugging the number of arms in the standard regret bound for UCB~\cite{bubeck_regret_2012}. Although the regret for this approach is bounded, the significant limitations are that the bound is exponential in the state and action space and, compared to SPSA, it cannot track changes in the system.

\vspace{-3mm}
\section{Example: Covert Federated Learning for Hate-Speech Classification}\label{sec:numerical}
We demonstrate the proposed covert optimization framework on a numerical experiment for hate-speech classification using federated learning in the presence of an eavesdropper. An eavesdropper spoofs as a client and misuses the optimal weights to generate hate speech, which goes undetected by the classifier. {The detailed motivation and experimental setup can be found in~\cite{jain_controlling_2024}.}
A balanced subset of the civil comments toxicity dataset by Jigsaw AI is used, which has comments along with annotations for whether the comment is toxic or not The federated learning setup consists of $N_{c} = 35$ clients, each having $N_d = 689$ data points. A fully connected neural network attached to a pre-trained transformer is trained with a cross-entropy loss to classify the comments as toxic or not. The accuracy is reported on a balanced validation dataset~{\footnote{The results are reproducible and can be found on the Github repository: github.com/aditj/CovertOptimization. The repository also contains links to the dataset, the complete set of experimental parameters, and {a supplementary document with additional benchmarks and illustrations}.}}.  

We consider $\numsuccessfulupdates=45$ successful gradient steps and $\numqueries=100$ queries, and the oracle levels are based on client participation. Each client participates in a Markovian fashion with a probability of staying connected or not connected as $0.8$. $\oraclelevels = 3$ oracle states correspond to the minimum number of participating clients $\statespace^\oraclesymbol=[1 = 1,2 = 12,3 = 24]$. We consider $\numprices = 3$ incentive levels as $\{1,2,3\}$. The number of samples each client contributes in each round depends on the incentive, as [10\%,40\%,80\%] of $N_d$ for the respective incentives. We consider a round successful if the number of samples exceeds $4000$.  The emperical success probabilities are $\succfunction(\oraclestate,\price) = [[0,0.1,0.2],[0.1,0.2,0.6],[0.3,0.6,0.9]]$. {The functions $\buffercost$ and $\oraclecost$ in~\eqref{eq:learningcost} are quadratic in $\learnerstate$ and $\oraclestate$, respectively.} { This satisfies assumptions R1, R2, R3. The emperical success probabilities along with the resulting cost function of~\eqref{eq:learningcost} ensure that R4, R5, R6 are satisfied for $\idfactorone_{\learnerstate,\learnerstate^\prime,\action} = \idfactortwo_{\learnerstate,\learnerstate^\prime,\action}\leq 1$.  } The queue cost is $\queuecost(\learnerstate) \propto \learnerstate^4$. The optimal stationary policy with the threshold structure is obtained using SPSA with $\spsastepsize_\timeindex = 0.01$, $\perturbation = 0.1$, and $\spsatimehorizon = 3000$ episodes.

\begin{table}[h!]
\vspace{2mm}
    \centering
    \begin{tabular}{cccc}
    \hline
     Type of Policy  & Learner Acc. & Eaves. Acc. & Incentive \\\hline
    Optimal Policy & 90\% & 54\% &  254 \\
     {Optimal Policy from~\cite{jain_controlling_2024}}&{ 89\%} & {53\%} & {290 }\\
     Greedy Policy & 91\% & 89\%  & 300 \\
     Random Policy & 52\% & 53\% & 190 \\
     \hline
    \end{tabular}
    \caption{The optimal stationary policy with a threshold structure outperforms greedy policy by $35\%$ on eavesdropper accuracy and random policy by $38\%$ on learner accuracy.}
        \label{tab:exp1}
\vspace{-8mm}
\end{table}

The results are averaged for $N_{mc} = 100$ runs and reported in  Table~\ref{tab:exp1}. {The greedy policy learns first with a maximum incentive, and random policy uniformly samples from the action space.} The optimal policy is better than the greedy policy in terms of the eavesdropper accuracy corresponding to the maximum a posteriori trajectory of~\eqref{eq:eavesestimate}. The optimal policy outperforms the random policy on learner accuracy. { The learner saves $~14\%$ incentive spent compared to the greedy policy. We also benchmark against the optimal policy from~\cite{jain_controlling_2024} with constant incentivization ($\price_\timeindex=3$) and similar to the greedy policy, the accuracies are comparable, but the optimal policy of this paper improves incentive expenditure by $12\%$.}
 \vspace{-3mm}
\section{Conclusion}\label{sec:conc}
{The proposed MDP} framework solves the learner's problem of dynamically optimizing a function {by querying and incentivizing a stochastic oracle} and obfuscating an eavesdropper by switching between two stochastic gradients. Using interval dominance, we prove structural results on the monotone threshold nature of the optimal policy. {In our numerical experiments, the optimal stationary policy with the threshold structure outperformed the greedy policy on the eavesdropper accuracy and the incentive spent.} { In future work, the problem of obfuscating sequential eavesdroppers can be formulated as a Bayesian social learning problem, where initially 
 the eavesdropper is obfuscated maximally to make it stop participating and its departure provides an indication to the subsequent eavesdroppers that the learner is obfuscating. Hence, the eavesdroppers can eventually be made to herd, forming an information cascade so that they don't eavesdrop anymore, regardless of whether the learner is learning or not.  }
 \vspace{-7mm}
\bibliographystyle{abbrv}
\bibliography{references}
\end{document}

%% file: input.tex
\newtheorem{theorem}{Theorem}
\newtheorem{definition}{Definition}
\newcommand{\function}{f}

\newcommand{\functionvar}{x}
\newcommand{\optimalpoint}{\functionvar^*}
\newcommand{\stepsize}{\mu}
\newcommand{\argmin}{\arg\min}
\newcommand{\timeindex}{k}
\newcommand{\query}{q}
\newcommand{\gradient}{\nabla}
\newcommand{\noise}{\eta}
\newcommand{\response}{r}
\newcommand{\numsuccessfulupdates}{M}
\newcommand{\numqueries}{N}
\newcommand{\learnerstate}{b}
\newcommand{\actionspace}{\mathcal{U}}
\newcommand{\action}{u}
\newcommand{\learnaction}{a}
\newcommand{\policy}{\pi}
\newcommand{\numinteractions}{T}

\newcommand{\statespace}{\mathcal{Y}}
\newcommand{\statevar}{y}
\newcommand{\learnersymbol}{B}
\newcommand{\oraclesymbol}{O}
\newcommand{\oraclelevels}{R}
\newcommand{\oraclestate}{o}

\newcommand{\R}{\mathbb{R}}
\newcommand{\learnerestimate}{\hat{\functionvar}}
\newcommand{\probability}{\mathds{P}}
\newcommand{\transitionqueue}[1]{\mathcal{P}^{\oraclestate,\oraclestate^{\prime}}_{\learnerstate^{#1}}}

\newcommand{\oracleevolutionprob}[2]{\Delta({#2}|{#1})}
\newcommand{\valuefn}{V}
\newcommand{\qfunction}{Q}

\newcommand{\dpindex}{n}
\newcommand{\cost}{c}
\newcommand{\queuecost}{d}
\newcommand{\optpolicy}{\policy^*}

\newcommand{\indicator}{\mathds{1}}

\newcommand{\thresholdlower}{\Bar{\learnerstate}}

\newcommand{\expectation}{\mathds{E}}
\newcommand{\policyspace}{\Pi}
\newcommand{\optaction}{\action^{*}}

\newcommand{\sigmoidalparameter}{\tau}
\newcommand{\approxpolicy}{\hat{\policy}}
\newcommand{\perturbation}{\delta}

\newcommand{\lipschitz}{\gamma}
\newcommand{\functionvaralt}{z}
\newcommand{\dimfunc}{d}
\newcommand{\noisevariance}{\sigma^2}

\newcommand{\succfunction}{\Gamma}
\newcommand{\price}{i}
\newcommand{\numprices}{{n_{\price}}}
\newcommand{\previousincentivesum}{I_\dpindex}
\newcommand{\succreply}{s}

\newcommand{\truetrajectory}{\mathcal{J}_1}
\newcommand{\falsetrajectory}{\mathcal{J}_2}
\newcommand{\trajectoryprobability}{\delta}

\newcommand{\thresholdcount}{F}

\newcommand{\obfuscatingestimate}{\hat{\functionvaralt}}
\newcommand{\buffercost}{\psi_1}
\newcommand{\oraclecost}{\psi_2}
\newcommand{\spsastepsize}{\phi}
\newcommand{\parametersspsa}{\Theta}
\newcommand{\spsatimeindex}{i}
\newcommand{\spsaapproxgradient}{\Tilde{\gradient}C}
\newcommand{\spsatimehorizon}{H}

\newcommand{\increasing}{\uparrow}

\newcommand{\idfactorone}{\alpha}
\newcommand{\idfactortwo}{\beta}
\newcommand{\idfactorthree}{\gamma}
\newcommand{\timeindexalt}{t}

%% file: figure.tex
\tikzset{every picture/.style={line width=0.75pt}} 

\begin{tikzpicture}[x=0.75pt,y=0.75pt,yscale=-1,xscale=1]

\draw   (353,145) -- (430,145) -- (430,192) -- (353,192) -- cycle ;
\draw    (352,180) -- (223,180) ;
\draw [shift={(221,180)}, rotate = 360] [color={rgb, 255:red, 0; green, 0; blue, 0 }  ][line width=0.75]    (10.93,-3.29) .. controls (6.95,-1.4) and (3.31,-0.3) .. (0,0) .. controls (3.31,0.3) and (6.95,1.4) .. (10.93,3.29)   ;
\draw    (348,161.5) -- (219,161.5) ;
\draw [shift={(350,161.5)}, rotate = 180] [color={rgb, 255:red, 0; green, 0; blue, 0 }  ][line width=0.75]    (10.93,-3.29) .. controls (6.95,-1.4) and (3.31,-0.3) .. (0,0) .. controls (3.31,0.3) and (6.95,1.4) .. (10.93,3.29)   ;
\draw   (130,146) -- (216,146) -- (216,193) -- (130,193) -- cycle ;
\draw  [color={rgb, 255:red, 255; green, 0; blue, 0 }  ,draw opacity=1 ][dash pattern={on 2.53pt off 3.02pt}][line width=2.25]  (222,80) -- (351,80) -- (351,131.5) -- (222,131.5) -- cycle ;

\draw (356,152) node [anchor=north west][inner sep=0.75pt]   [align=left] {Oracle ($\Tilde{\nabla}\function$)};
\draw (305,109.9) node [anchor=north west][inner sep=0.75pt]    {$\query_{\timeindex}$};
\draw (330,208.9) node [anchor=north west][inner sep=0.75pt]    {$\succreply_\timeindex$};
\draw (322,193.9) node [anchor=north west][inner sep=0.75pt]    {$\response_{\timeindex}$};
\draw (315,140.9) node [anchor=north west][inner sep=0.75pt]    {$\learnaction_{\timeindex}$};
\draw (384,168.4) node [anchor=north west][inner sep=0.75pt]    {$\oraclestate_{\timeindex}$};
\draw (225,138.5) node [anchor=north west][inner sep=0.75pt]   [align=left] {Type of Query};
\draw (257,109.5) node [anchor=north west][inner sep=0.75pt]   [align=left] {Query};
\draw (232,190.5) node [anchor=north west][inner sep=0.75pt]   [align=left] {Noisy Gradient};
\draw (225,206.5) node [anchor=north west][inner sep=0.75pt]   [align=left] {Response Success};
\draw (150,151) node [anchor=north west][inner sep=0.75pt]   [align=left] {Learner};
\draw (166,166.4) node [anchor=north west][inner sep=0.75pt]    {$\learnerstate_{\timeindex}$};
\draw (305,87) node [anchor=north west][inner sep=0.75pt]    {$\price_{\timeindex}$};
\draw (244,87.5) node [anchor=north west][inner sep=0.75pt]   [align=left] {Incentive};
\draw (210,65) node [anchor=north west][inner sep=0.75pt]   [align=left] {\textit{Visible to Eavesdropper}};

\end{tikzpicture}